\documentclass{article}

\usepackage{microtype}
\usepackage{graphicx}
\usepackage[caption=false,font=footnotesize]{subfig}
\usepackage{booktabs} 

\usepackage{hyperref}

\usepackage[accepted]{icml2023}

\usepackage{amsmath}
\usepackage{amssymb}
\usepackage{mathtools}
\usepackage{amsthm}

\usepackage[capitalize,noabbrev]{cleveref}
\usepackage{threeparttable}
\usepackage{pdfpages}
\usepackage{multirow}

\theoremstyle{plain}
\newtheorem{theorem}{Theorem}[section]
\newtheorem{proposition}[theorem]{Proposition}
\newtheorem{lemma}[theorem]{Lemma}

\theoremstyle{definition}

\theoremstyle{remark}

\usepackage[textsize=tiny]{todonotes}

\icmltitlerunning{Transformers are Universal Predictors}

\begin{document}

\twocolumn[
\icmltitle{Transformers are Universal Predictors}




\begin{icmlauthorlist}
\icmlauthor{Sourya Basu}{lab,ece}
\icmlauthor{Moulik Choraria}{lab,ece}
\icmlauthor{Lav R.\ Varshney}{lab,ece}
\end{icmlauthorlist}

\icmlaffiliation{lab}{Coordinated Science Laboratory, University of Illinois Urbana-Champaign}
\icmlaffiliation{ece}{Electrical and Computer Engineering, University of Illinois Urbana-Champaign}

\icmlcorrespondingauthor{}{\{sourya, moulikc2, varshney\}@illinois.edu}

\icmlkeywords{Causal inference, discontinuity regression, dynamic programming, quantization theory, reinforcement learning}

\vskip 0.3in
]



\printAffiliationsAndNotice{}  

\begin{abstract}
We find limits to the Transformer architecture for language modeling and show it has a universal prediction property in an information-theoretic sense. We further analyze performance in non-asymptotic data regimes to understand the role of various components of the Transformer architecture, especially in the context of data-efficient training.  We validate our theoretical analysis with experiments on both synthetic and real datasets.
 \end{abstract}

\section{Introduction}
\label{sec: intro}
Language models that aim to predict the next token or word to continue/complete a prompt have their origins in the work of \citet{Shannon1948, Shannon1950b, Shannon1951}.  In recent years, neural language models have taken the world by storm, especially the Transformer architecture~\cite{vaswani2017attention}.  Following work on other neural network architectures \cite{Cybenko1989}, one can show that the Transformer architecture has a \emph{universal approximation} property for sequence-to-sequence functions \cite{YunBRRK2020}.

Transformer architectures have excellent performance and parallelization capability on natural language processing (NLP) tasks, becoming central to several state-of-the-art models including GPT-4 \cite{OpenAI2023} and PaLM 2 \cite{Anil_ea2023}.  Such large language models (LLMs) are not only very good at the statistical problem of predicting the next token, but also have emergent capabilities in tasks that seemingly require higher-level semantic ability \cite{Wei_ea2022}.
Moreover, transformer-based architectures have achieved tremendous attention on domains beyond NLP, such as images~\cite{dosovitskiy2020image}, audio~\cite{li2019neural}, reinforcement learning~\cite{chen2021decision}, and even multi-modal tasks~\cite{jaegle2021perceiver}. Transformers also show cross-domain transfer learning capabilities, i.e., models trained on NLP tasks show good performance when fine-tuned for non-NLP tasks such as image processing. In this sense, the Transformer architecture is said to have a \emph{universal computation} property~\cite{lu2021pretrained}, reminiscent of predictive coding hypotheses of the brain that posit one basic operation in neurobiological information processing \cite{NEURIPS2022_5b5de852}. 

The basic predictive workings of Transformers and previous findings of universal approximation and computation properties motivate us to ask whether they also have a \emph{universal prediction} property in the information-theoretic sense \cite{FederMG1992, WeissmanM2001}, which itself is well-known to be intimately related to universal data compression \cite{MerhavF1998}.  As far as we know, the predictive capability of Transformers has not been studied in an information-theoretic sense, cf.~\citet{GurevychKS2022}.

We investigate not only the underlying mathematical principles that govern the performance of Transformers,  but also aim to find limitations to their learning capabilities.
We show that Transformers are indeed universal predictors, i.e.\ they can achieve information-theoretic limits asymptotically in the amount of data available. We also analyze their performance in the finite-data regime by understanding the role of various components of the Transformer architecture, providing theoretical explanations wherever applicable.

To summarize our main results, we find the limits to performance of Transformers and show they are optimal predictors. Our limits only assume the Markov nature of data and are otherwise universal. Moreover, we analyze the role of the major components of a Transformer and provide better understanding and directions for their data-efficient training. Finally, we validate our theoretical analysis by performing experiments on both synthetic and real datasets.
\section{Definitions and Preliminaries}\label{sec: defn}
\subsection{Finite-State Markov Processes (FSMPs)}\label{subsec: fsmps}
Let $x = \{x_1, \ldots x_n\} \in \mathcal{X}^n$ be sequential data, where $\mathcal{X}$ is some finite set. The state sequence of an FSMP, $s = \{s_1, \ldots, s_{n-1}\}$, is generated recursively according to $s_i = g(x_i, \ldots, x_{(i-k+1)_+})$, where $g(\cdot)$ is the state function for this FSMP, $s_i \in \mathcal{S}$, and $(i)_+ = \max{\{i, 1\}}$. An FSMP has a predictor function $f(\cdot)$ that outputs a probability distribution over possible values for $x_{i+1}$, i.e.\ $\hat{x}_{i+1} = f(s_i) \in \mathbb{R}^{|\mathcal{X}|}$, $\sum_{j} f(s_i)_j = 1$. Hence, a FSMP is given by the pair $(g, f)$.

\subsection{Transformers}\label{subsec: transformers}
A Transformer architecture consist of three components placed in a series: an input embedding layer, multiple attention layers, and an output projection matrix. Let us describe each of the subcomponents of a Transformer.

\paragraph{Input embedding layer} Let the input be sequential data $x = \{x_1, \ldots x_{n-1}\} \in \mathcal{X}^{n-1}$. The embedding layer $E$ processes the input sequence individually to give a sequence $z = \{z_1, \ldots z_{n-1}\} \in \mathbb{R}^{(n-1) \times d_{in}}$. 

\paragraph{Attention layer} The attention layer further consists of two subcomponents: the self attention mechanism and the position-wise feedforward network.

The masked self-attention layer takes input $X = [x_1, x_2, \ldots, x_{n-1}]^\text{T} \in \mathbb{R}^{(n-1)\times d_{in}}$. Queries ($Q$), keys ($K$), and values ($V$) are computed from $X$ by multiplying with three corresponding matrices as $Q=XW_{Q}, K=XW_{K}$, and $V=XW_{V}$, where each matrix $W_{Q}, W_{K}$, and $W_{V}$ is of dimension $d_{in} \times d_{model}$. The output of the self-attention layer, $H^{(0)} = [h^{(0)}_1, h^{(0)}_2, \ldots, h^{(0)}_{n-1}]$ is given by 
\begin{align*}
    &H^{(0)} \\ &= \text{Attention}(Q, K, V)
            = \text{softmax}(\text{mask}_{-\infty}(\frac{QK^{\text{T}}}{\sqrt{d_{model}}}) )V,
\end{align*}
where mask$_{-\infty}$ is the causal binary mask used to preserve the auto-regressive property of language modeling by ensuring $h^{(0)}_i$ depends only on $x_{j \leq i}$ by setting all the entries in the matrix $QK^{\text{T}}$ corresponding to connections to $x_{j>i}$ to $-\infty$. 

This attention mechanism is extended to multi-head attention with $m$ heads by simply dividing the inputs of dimension $d_{model}$ into $m$ sub-parts and computing attention separately and then concatenating them.

Masked self-attention followed by position-wise feedforward network, $FFN$, gives a Transformer decoder layer $H^{(1)} = [h^{(1)}_1, h^{(1)}_2, \ldots, h^{(1)}_{n-1}]$.

\paragraph{Output projection matrix} The output projection matrix $W_p$ takes as input a sequence of dimension $\mathbb{R}^{(n-1)\times d_{model}}$ and outputs probabilities on the output space of dimension $\mathbb{R}^{(n-1) \times |\mathcal{V}|}$, where $\mathcal{V}$ is the vocabulary space.

An $L$-layered Transformer decoder consists of an embedding layer, followed by $L$ attention layers, followed by an output projection matrix. A single-layered Transformer decoder is shown in Fig.~\ref{fig: transformer_decoder_layer}.

\begin{figure}
    \centering
    \includegraphics[width=0.45\textwidth]{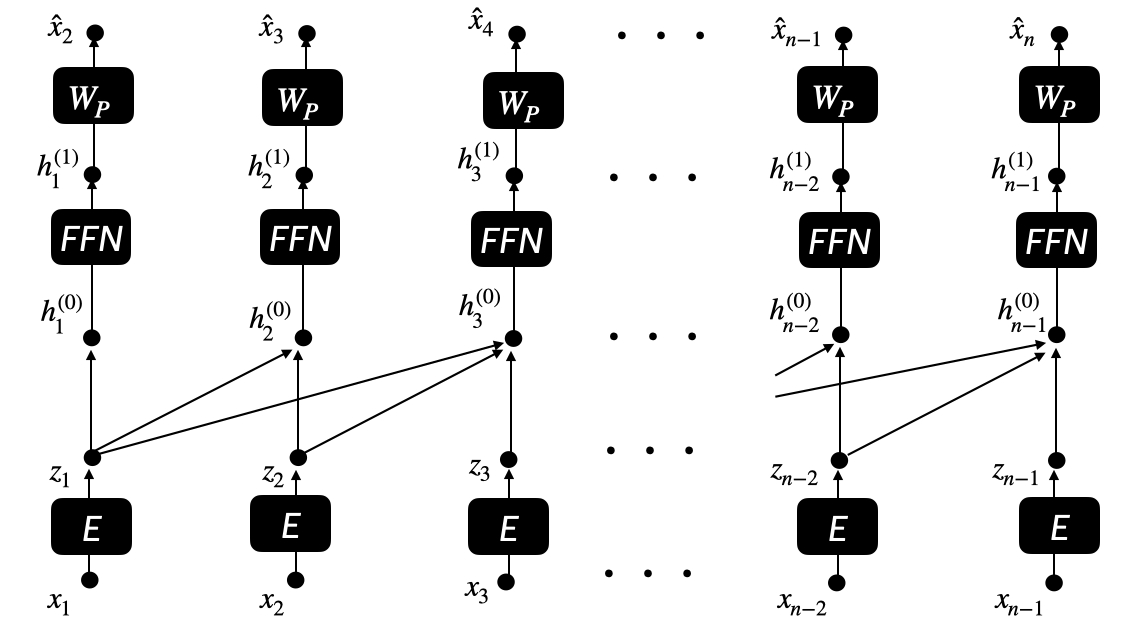}
    \caption{Single-layered Transformer decoder with attention span, $k$, $=3$.}
    \label{fig: transformer_decoder_layer}
\end{figure}

\section{Performance Limits of Transformers}\label{sec: universal_predictability}
Here we provide theoretical limits to the performance of the Transformer architecture. First, we show that the Transformer architecture can be viewed as an approximate FSMP. 

\subsection{Transformers as approximate FSMPs}\label{subsec: transformers_as_fsmps}
An FSMP as defined in Sec.~\ref{subsec: fsmps} is given by a function pair $(g, f)$, where $g$ is a state function that first \emph{aggregates} certain past observations $(x_j, \ldots, x_{i-1})$, where $j < (i-1)$ is a choice for the $g$ function and $f$ is a probability function from the states given by $g$ to $\mathbb{R}^{|\mathcal{V}|}$.

In an $L$-layered Transformer, we model the $g$ function by the embedding layer $E$ followed by a sequence of $L$ attention layers. The $l$th attention layer computes the weighted sum of past observations $(h^{(2l)}_j, \ldots, h^{(2l)}_{i-1})$ followed by an $FFN$. Note that the weighted sum in the attention mechanism is performed in the higher dimension $d_{model}$, which can retain as much information as concatenation in lower dimension if $d_{model}$ is large enough. The output of the embedding layer followed by $L$-attention layers can be seen as approximating the $g$ function, call it $\Bar{g}$. 

This is followed by the output projection matrix $W_P$, which can be seen as approximating the output probability function $f$, call it $\Bar{f}$. Fig.~\ref{fig: transformer_as_FSMP} shows a single-layer Transformer, comparing its components with that of an FSMP.

\begin{figure}
    \centering
    \includegraphics[width=0.35 \textwidth]{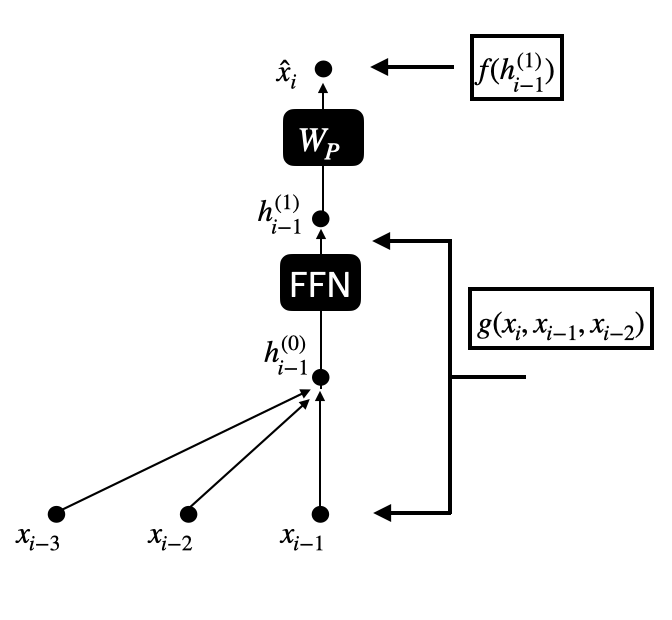}
    \caption{
    A single Transformer decoder layer with attention span equals to three ($k=3$) and its comparison to a finite-state Markov predictor (FSMP). The output of the Transformer decoder, $h_{i-1}^1$ can be interpreted as the output of the state function, $g(x_i, x_{i-1}, \ldots, x_{i-k+1})$, of a FSMP. The matrix $W_P$ can be interpreted as the predictor function $f(\cdot)$ of a FSMP that takes the state as input and outputs a probability distribution over the vocabulary space. The embedding layer is ignored for brevity.}
    \label{fig: transformer_as_FSMP}
\end{figure}

\subsection{Theoretical Limits}\label{subsec: theoretical_limits}
Here, we use the similarity between Transformers and FSMPs to find the limits of the Transformer architecture. First we provide the training setup and loss criterion for our results. Note that the only assumption we use in the following data generation process is of Markovity and hence, the results obtained are general otherwise. We first describe the dataset and the loss criterion.

\paragraph{Dataset} We consider the train dataset $\mathcal{D}_{Train} = \{x_1,\ldots,x_n; y_1, \ldots, y_n\}$ of size $n$ and test dataset $\mathcal{D}_{Test}=\{x_{n+1}, \ldots, x_{n+m}; y_{n+1},\ldots,y_{n+m}\}$ of size $m$. In the datasets, $\{x_i\}_{i=1}^{n+m}$ is an arbitrary sequence with distribution $p(x_1, \ldots, x_{n+m})$ and $y_i$ is generated conditioned on $g(x_i, \ldots, x_{(i-l+1)_+})$, where $g$ is a state function.
Thus, we consider a sequence of Markov order $l$. For our theoretical analysis, we use binary vocabulary, i.e.\ $x_i, y_i \in \{0,1\}$, and will let $m \to \infty$.

\paragraph{Loss criterion} A model is trained on $\mathcal{D}_{Train}$ with some state function $\Bar{g}$ to obtain an order $k$ stationary Markov estimator  $p^n_{\theta_k}(\cdot)$, which is the $f$ function in the context of an FSMP.

The train loss on $\mathcal{D}_{Train}$ is given by $\mathcal{L}_{Train}^{(l,k,n)}(\mathcal{D}_{Train}) = -\frac{1}{n}\sum_{i=1}^n \log{(p^n_{\theta_k}(y_i|\Bar{g}(x_{\leq i})))}.
$

Denote the loss on $\mathcal{D}_{Test}$ by $\mathcal{L}_{Test}^{(l,k,m,n)}(\mathcal{D}_{Test}) = -\frac{1}{m}\sum_{i=1}^m \log{(p^n_{\theta_k}(y_{i+n}|\Bar{g}(x_{\leq i+n})))}$. Following ~\cite{ManningS1999}, we consider languages to show stationary, ergodic properties. Now, taking $m \to \infty$, and using stationarity of $p(\cdot)$, we define $\mathcal{L}_{Test}^{(l,k,n)}(\mathcal{D}_{Test}) = \lim_{m \to \infty} \mathcal{L}_{Test}^{(l,k,m,n)}(\mathcal{D}_{Test})$ to get
\begin{align}\label{eqn: test_loss}
&\mathcal{L}_{Test}^{(l,k,n)}(\mathcal{D}_{Test}) =\nonumber \\
&-\sum_{(y_k, x_k,\ldots, x_1) \in \{0,1\}^k}p(y_k, x_k\ldots, x_1) \log{(p^n_{\theta_k}(y_k|\Bar{g}(x_{k},\ldots,x_1)))}.
\end{align}
For fixed data Markov order $l$ with known state function $g$ and estimator Markov order $k$, suppose the minimum loss is obtained by an estimator $p^{n,*}_{\theta_k}$. Then, this minimum loss is 
\begin{align}\label{eqn: optimal_test_loss}
&\mathcal{L}_{Test}^{(l,k,n),*}(\mathcal{D}_{Test}) =\nonumber\\
-&\sum_{(y_k, x_k,\ldots, x_1) \in \{0,1\}^{k+1}}p(y_k, x_k,\ldots, x_1) \log{(p^{n,*}_{\theta_k}(y_k|g(x_{k},\ldots,x_1)))}.
\end{align}

\paragraph{Bayesian estimator} 
An order-$k$ Bayesian estimator with a given state function $g$ is of the form 
\begin{align}\label{eqn: k_estimator}
p_{\theta_k}^n(Y_{k+i} = y_k|g(X_{k+i}, \ldots, X_{1+i}) = g(x_{k}, \ldots, x_{1}))\nonumber \\ =\begin{cases}\frac{N{(y_k; g(x_{k}, \ldots, x_{1}))}}{N{(g(x_{k}, \ldots, x_{1}))}} & \text{if } k>0,\\
\frac{N{(y_0)}}{n} & \text{if } k=0,
\end{cases}
\end{align}
where $N{(y_k; g(x_{k}, \ldots, x_{1}))}$ denotes the number of counts of $Y_{k+i} = y_k$ and $g(X_{k+i}, \ldots, X_{1+i}) = g(x_{k}, \ldots, x_{1})$ for $0 \leq i \leq (n-k)$ in the train set, $\mathcal{D}_{Train}$. Similarly, $N{(g(x_{k}, \ldots, x_{1}))}$ is the number of occurrences of the subsequence  $g(X_{k+i}, \ldots, X_{1+i}) = g(x_{k}, \ldots, x_{1})$, $N(y_0)$ is the number of occurrences of $Y_i = y_0$ for $0 \leq i \leq (n-k)$, and $n$ is the size of the train dataset. 
Note that the order $l$ and the state function $g$ of the source is unknown to the estimator, hence it is not obvious what $k$ and $g$ should be chosen for best test performance.

Now we state and prove the theorem giving limits to the test loss obtained by any Transformer architecture with attention span $k$. Thereafter we show the obtained limits are indeed optimal. Proofs are  in Appendix~\ref{sec: proofs}.

\begin{theorem}\label{thm: thm_convergence}
Let $\{\mathcal{D}_{Train}, \mathcal{D}_{Test}\}$ be an order-$l$ Markov dataset with some state function $g$. Let $\Bar{g}$ be the state function of a Transformer such that $H(Y_m|g(X_m,\ldots,X_{1})) = H(Y_m|\Bar{g}(X_m,\ldots,X_{1}))$ for $m=k, l$, e.g. choose $\Bar{g}$ as the identity function and let $p_{\theta_k}^n$ be the order-$k$ Bayesian estimator in \eqref{eqn: k_estimator} with state function $\Bar{g}$. Let $\mathcal{L}_{Test}^{(l,k,n)}(\mathcal{D}_{Test})$ be the corresponding test loss of the Transformer. 

Then, if $l \leq k$,
\begin{align}\label{eqn: thm_l_lq_k}
    \lim_{n \to \infty} \mathcal{L}_{Test}^{(l,k,n),*}(\mathcal{D}_{Test}) = H(Y_l|X_l,\ldots,X_{1}), 
\end{align}
whereas, if $l > k$,
\begin{align}\label{eqn: thm_k_lq_l}
    \lim_{n \to \infty} \mathcal{L}_{Test}^{(l,k,n),*}(\mathcal{D}_{Test}) = H(Y_k|X_k,\ldots,X_{1}). 
\end{align}
Here $H(Y_m|X_m,\ldots,X_{1})$ is the conditional entropy of the distribution $p(y_m|x_m,\ldots,x_{1})$ for $m=l, k$. 

Moreover, $\mathcal{L}_{Test}^{(l,k,n)}(\mathcal{D}_{Test}) \to \mathcal{L}_{Test}^{(l,k,n),*}(\mathcal{D}_{Test})$ with convergence error $\mathcal{O}({\frac{|\mathcal{S}|}{n}})$ in expectation in both the cases, \eqref{eqn: thm_l_lq_k}, \eqref{eqn: thm_k_lq_l}, where $\mathcal{S}$ is the range of the state function $\Bar{g}$.
\end{theorem}

Now, we show that the limits obtained in Thm.~\ref{thm: thm_convergence} are optimal.

\begin{theorem}\label{thm: optimal_limit}
Let $\{\mathcal{D}_{Train}, \mathcal{D}_{Test}\}$ be an order-$l$ Markov dataset with some state function $g$. Then, let $h$ be some arbitrary function taking as input $(X_{i}, \ldots, X_{i+k})$, and outputs some probability $h(Y_{i+k}|X_{i}, \ldots, X_{i+k})$ over the space of $Y_{i+k}$. Then, the optimal cross-entropy loss obtained by $h$, $\sum_{(y_k, x_k,\ldots, x_1) \in \{0,1\}^k}p(y_k, x_k\ldots, x_1) \log{(h(y_k| x_k\ldots, x_1))}$ is greater than or equal to $H(Y_l|X_l, \ldots, X_1)$ if $l \leq k$, else, greater than or equal to $H(Y_k|X_k, \ldots, X_1)$.
\end{theorem}

\subsection{Understanding the limits}\label{subsec: understanding_limits}

\begin{figure}
\centering     
\includegraphics[width=55mm]{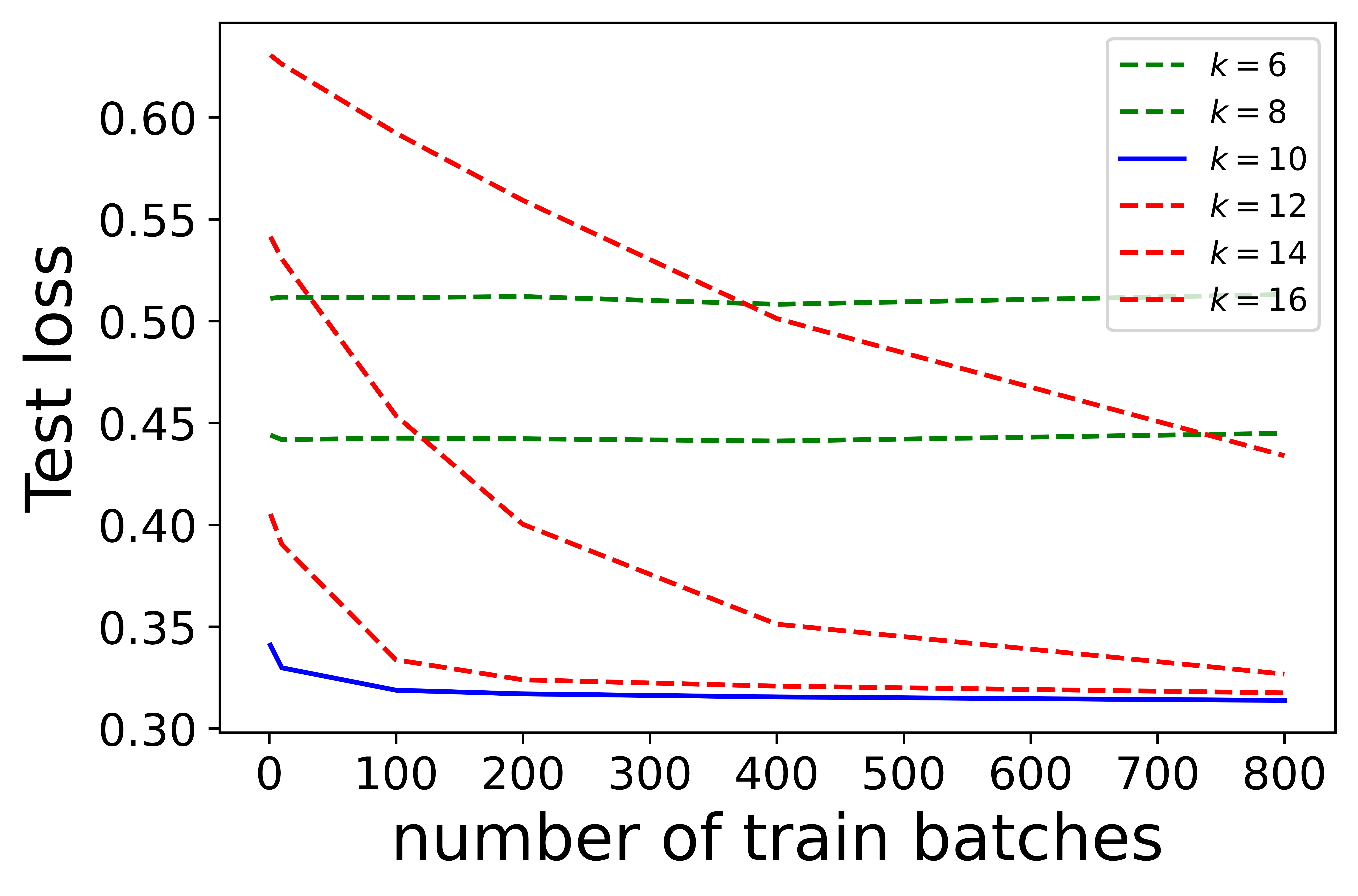} \\($l=10$)
\\

\includegraphics[width=55mm]{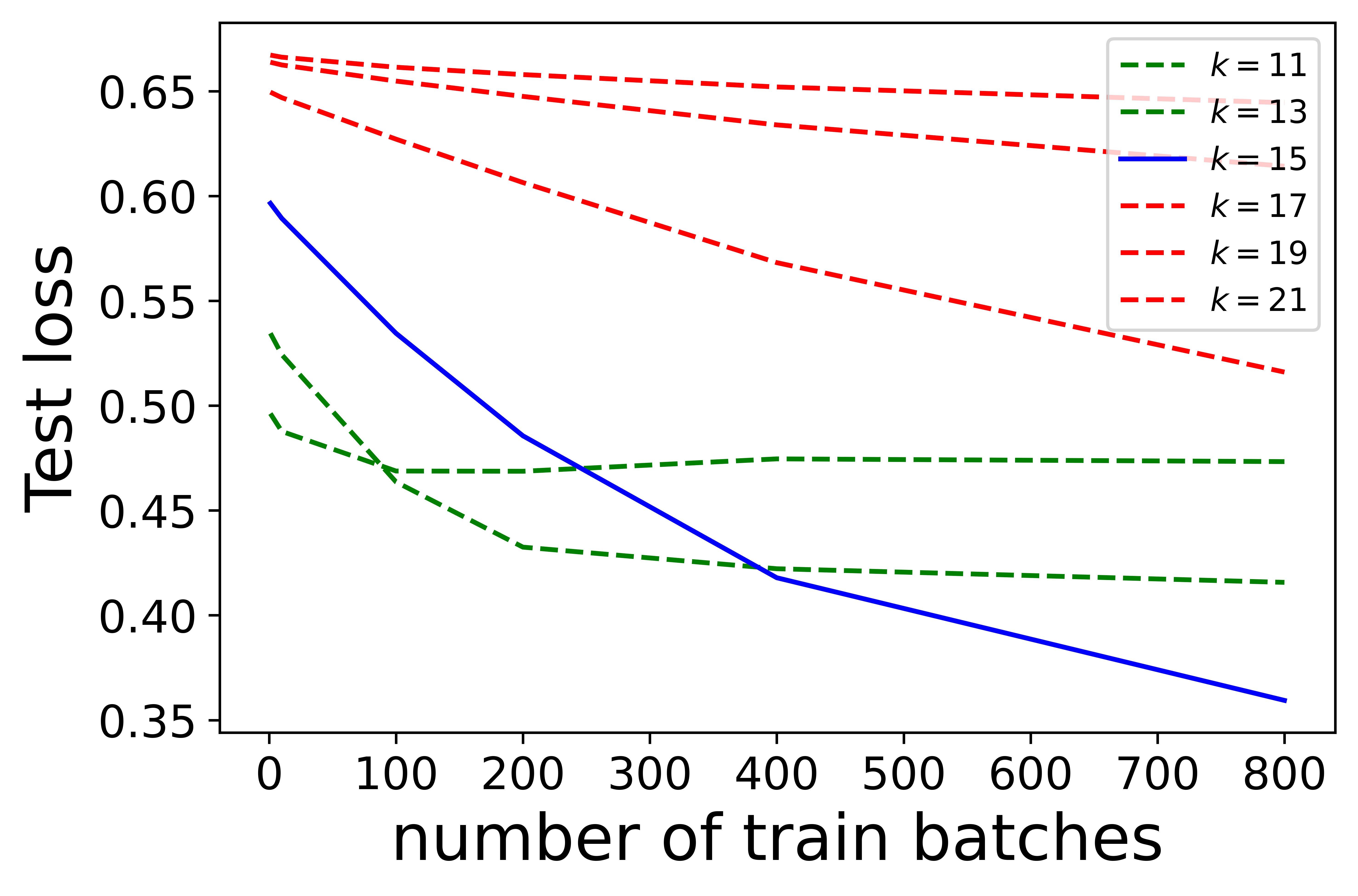} \\
($l=15$) \\

\includegraphics[width=55mm]{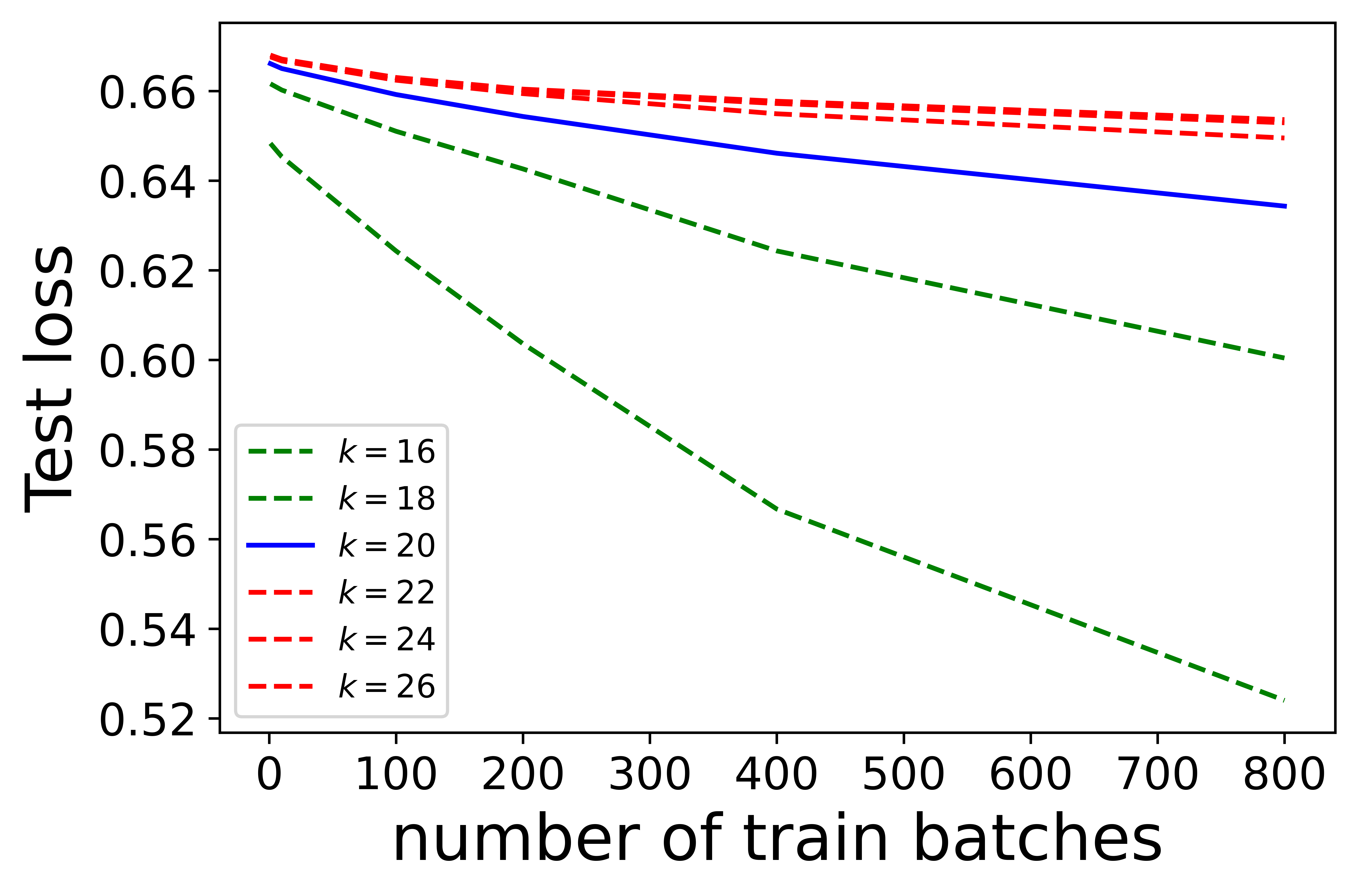}

($l=20$)

\caption{Performance of Finite State Markov Predictor on a Boolean synthetic dataset, where the output is 1 if the sum of previous $l$ observations is greater than a threshold, else 0. Plot illustrates the nature of the obtained limit in Thm.~\ref{thm: thm_convergence}. Smaller values of $k$ converges faster but to a higher value of test loss compared to larger values.}
\label{fig: fsmp}
\end{figure}

From the limits in Thm.~\ref{thm: thm_convergence}, we know that for $k < l$, increasing the value of $k$ improves the obtained limit. This is because conditioning reduces entropy and hence using the stationarity property $H(Y_k|X_{k},\ldots,X_1) \leq H(Y_{k-1}|X_{k-1},\ldots,X_1)$~\cite{CoverT1999}. But, for $k > l$, $H(Y_l|X_{l},\ldots,X_1) = H(Y_{k}|X_{k},\ldots,X_1)$. To better understand the limits in Thm.~\ref{thm: thm_convergence}, we show the performance of an FSMP with the aggregation function $g$ as the identity function with span of the input $k$, and the output probability function $f$ is set as the Bayesion predictor in \eqref{eqn: k_estimator}. We test the performance of the FSMP on a synthetic Boolean dataset, where the input is randomly generated binary values and the output at position $i$ is 1 if the sum of the past $l$ observations is greater than a threshold, else, 0. The main observation to note from Fig.~\ref{fig: fsmp} is that for smaller values of $k$, the convergence is faster compared to larger values, but, the value the losses converge to is much worse than for larger values of $k$.

The limits in Thm.~\ref{thm: thm_convergence} gives us two important design choices for efficiently designing and training Transformers. First, the choice of the state function $\Bar{g}$ and its attention span $k$. Second, the amount of data, $n$, to train on.

The choice of $\Bar{g}$ and $k$ must be such that it does not lose any information important to the output, i.e. $H(Y_m|g(X_m,\ldots,X_{1})) = H(Y_m|\Bar{g}(X_m,\ldots,X_{1}))$. One such $\Bar{g}$ can be simply the identity function. The corresponding $k$ must be chosen large enough, i.e. $k \geq l$. Note that for all such $\Bar{g}$ and $k$ the optimal value remains the same as $n \to \infty$, but, in the finite data regime, one needs to choose appropriate $\Bar{g}$ and $k$ to get the best performance. In particular, one needs to choose $\Bar{g}$ and $k$ that has the smallest range $\mathcal{S}$ while also satisfying $H(Y_m|g(X_m,\ldots,X_{1})) = H(Y_m|\Bar{g}(X_m,\ldots,X_{1}))$. Thus, it indicates the need for sparsity in the design.

For the second choice, $n$, while it may seem obvious that training on larger data would give better performance, we show that with increase in data, the improvement in performance is significant, even with training for the same number of training steps. Moreover, this also motivates an augmentation technique that helps statistically explain the benefit gained from using relative positional encodings~\cite{shaw2018self}.

In Appendix \ref{sec: finite_data_analysis}, we further aim to understand Transformers and data-efficient training.  In Appendix \ref{sec: experiments}, we provide further experiments on real-world data.



\section{Conclusion}\label{sec: conclusion}
Here, we prove that the Transformer architecture yields universal predictors in the information-theoretic sense, which may help explain why this architecture seems to be a kind of state-of-the-art universal computation system. We further analyzed the different components of the Transformer architecture, such as attention weights and positional encodings.

Going forward, it may be of interest to take inspiration from information-theoretic work in universal prediction \cite{MerhavF1998} to inspire novel neural architectures.

\bibliography{refs}
\bibliographystyle{icml2023}

\newpage
\appendix
\onecolumn

\appendix
\section{Proofs}\label{sec: proofs}
\begin{lemma}\label{lem: entropy_eq}
$H(Y_l|g(X_l,\ldots,X_{1})) = H(Y_l|X_l,\ldots,X_{1})$ for the data generation process in Sec.~\ref{sec: universal_predictability}.
\end{lemma}
\begin{proof}
We know $Y_l-g(X_l,\ldots,X_{1})-(X_l,\ldots,X_{1})$ forms a Markov chain, hence, $H(Y_l|g(X_l,\ldots,X_{1}), X_l,\ldots,X_{1}) = H(Y_l|g(X_l,\ldots,X_{1}))$, But, since $g(X_l,\ldots,X_{1})$ is a function of $(X_l,\ldots,X_{1})$, we have $H(Y_l|g(X_l,\ldots,X_{1}), X_l,\ldots,X_{1}) = H(Y_l|X_l,\ldots,X_{1})$.
\end{proof}

\begin{proof}[Thm.~\ref{thm: thm_convergence}]
First we consider the case where $l=0, k=0$ and understand the behavior of $\mathcal{L}_{Test}^{(l,k,n)}(\mathcal{D}_{Test})$ and $\mathcal{L}_{Test}^{(l,k,n),*}(\mathcal{D}_{Test})$ as a function of $n$. Then, we will use the results from these cases to understand more general cases varying both $l$ and $k$.

\paragraph*{Case $l=0, k=0$} Consider order $0$ estimator in \eqref{eqn: k_estimator}. We first compute $\mathcal{L}_{Test}^{(0,0,n),*}(\mathcal{D}_{Test})$ by writing 
\begin{align}\label{eqn: loss_func_l_eq_k}
\mathcal{L}_{Test}^{(0,0,n)}(\mathcal{D}_{Test}) = \sum_{y_0 \in \{0,1\}} p(y_0)\log{\frac{1}{p(y_0)}} + D(p||p_{\theta^n_0}),
\end{align}
where $D(p||p_{\theta^n_0}) = \sum_{y_0 \in \{0,1\}} p(y_0)\log{\frac{p(y_0)}{p_{\theta^n_0}(y_0)}}$ is the KL divergence between $p$ and $p_{\theta^n_0}$. Further, from basic information-theoretic inequalities we know that $D(p||q) \geq 0$ with the equality holding only when $p=q$ \cite{CoverT1999}. Thus, we have $\mathcal{L}_{Test}^{(0,0,n),*}(\mathcal{D}_{Test}) = H(Y_0)$, where $H(Y_0) = \sum_{y_0 \in \{0,1\}} p(y_0)\log{\frac{1}{p(y_0)}}$ is the entropy of the distribution $\{p(y_0): y_0 \in \{0,1\}\}$. Moreover, we have $p^{n,*}_{\theta_0}(y_0) = p(y_0)$ for $y_0 \in \{0,1\}$.

Now, we show that $D(p||p_{\theta^n_0}) \to 0$ as $n\to \infty$ and further analyze its convergence rate. Using concentration inequalities, it follows that $p_{\theta^n_0} \to p(y_0)$ with rate $\mathcal{O}(\frac{1}{\sqrt{n}})$ with high probability. For brevity, we do not provide the calculations and use results from \citet{CoverS1977} that show the convergence rate to be $\mathcal{O}(\frac{1}{\sqrt{n}})$ in expectation. Now we show that this implies $D(p(y_0)||p_{\theta^n_0}) \to 0$ with an approximate rate $\mathcal{O}(\frac{1}{n})$. To this end, take $p_{\theta^n_0} = p(y_0)+t$, where $t = \mathcal{O}(\frac{1}{\sqrt{n}})$. Say, $p(1) = p = p(y_0)$ and hence, $p(0) = 1-p$. Then, we have
\begin{align*}
D(p||p+t) &= p\log{\frac{p}{p+t}} + (1-p)\log{\frac{(1-p)}{(1-p-t)}} \nonumber \\
        &= -\frac{1}{\ln{2}}\left( p\ln{(1+\frac{t}{p})} + (1-p)\ln{(1-p-t)} \right) \mbox{.}\nonumber
\end{align*}
Assuming $n$ to be large enough such that $t \ll \min{\{p,1-p\}}$ and using the Taylor expansion of $\ln{(1-x)}$ up to second-order terms, we have
\begin{align}
D(p||p+t) &\approx \frac{t^2}{\ln{2}}\left(\frac{1}{p} + \frac{1}{1-p}\right).\nonumber
\end{align}
Thus, $D(p||p_{\theta^n_0}) \to 0$ with $\mathcal{O}(\frac{1}{n})$ in expectation.

Next, we generalize these optimal loss and convergence results for general $l$ and $k$ using results from the current case.

\paragraph*{Case $l \leq k$}
Here we have the evaluation loss $\mathcal{L}_{Test}^{(l,k,n)}(\mathcal{D}_{Test})$ as
\begin{align} \label{eqn: loss_func_l_less_k}
    \mathcal{L}_{Test}^{(l,k,n)}(\mathcal{D}_{Test}&) = H(Y_{l}|\Bar{g}(X_{l},\ldots,X_1)) + \nonumber \\ 
    & \sum_{(x_k,\ldots,x_{1})} p(x_k,\ldots,x_{1}) \times \\
    & D(p(y_k|\Bar{g}(x_k,\ldots,x_{1})) || p_{\theta_k}^n(y_k|\Bar{g}(x_k,\ldots,x_{1}))),
\end{align}
where $p_{\theta_k}^n(y_k|\Bar{g}(x_k,\ldots,x_{1}))$ is as defined in \eqref{eqn: k_estimator}. Since $D(p||q)) \geq 0$ with equality holding only when $p = q$, we have $\mathcal{L}_{Test}^{(l,k,n),*}(\mathcal{D}_{Test}) = H(Y_l|\Bar{g}(X_l,\ldots,X_{1})) = H(Y_l|g(X_l,\ldots,X_{1})) = H(Y_l|X_l,\ldots,X_{1})$, where $H(Y_l|g(X_l,\ldots,X_{1}))$ is the conditional entropy of the distribution $p(y_k|g(x_k,\ldots,x_{1}))$. The equality $H(Y_l|g(X_l,\ldots,X_{1})) = H(Y_l|X_l,\ldots,X_{1})$ follows directly and is proved in Lem.~\ref{lem: entropy_eq} for completeness. 

Moreover, we have $p^{n,*}_{\theta_k}(y_k|\Bar{g}(x_k,\ldots,x_{1})) = p(y_k|\Bar{g}(x_k,\ldots,x_{1}))$ for $(y_{k}, x_k,\ldots,x_{1}) \in \{0,1\}^{k+1}$. Note that here, for $l \leq k$, $\mathcal{L}_{Test}^{(l,k,n),*}(\mathcal{D}_{Test}) = H(Y_{l}|\Bar{g}(X_{l-1},\ldots,X_0))$ is independent of $k$. 

Now we look at the convergence rate of $D(p(y_k|\Bar{g}(x_k,\ldots,x_{1})) || p_{\theta_k}^n(y_k|\Bar{g}(x_k,\ldots,x_{1}))) \to 0$ as $n \to \infty$. Here we only perform approximate calculations as this is sufficient for present purposes. From the previous case with $l=0,k=0$, we know $D(p(y_0) || p_{\theta_0}^n(y_0)) \to 0$ as $\mathcal{O}(\frac{1}{n})$. For the convergence of $D(p(y_k|\Bar{g}(x_k,\ldots,x_{1})) || p_{\theta_k}^n(y_k|\Bar{g}(x_k,\ldots,x_{1})))$, we will only look at the indices $y_k$ that have prefix $\Bar{g}(x_{k},\ldots,x_1)$. We know that the number of such prefixes is approximately equal to $n p(\Bar{g}(x_k,\ldots,x_{1}))$ in expectation. Thus $D(p(y_k|\Bar{g}(x_k,\ldots,x_{1})) || p_{\theta_k}^n(y_k|\Bar{g}(x_k,\ldots,x_{1}))) \to 0$ as $\mathcal{O}(\frac{1}{n p(\Bar{g}(x_k,\ldots,x_{1}))})$. Thus, from \eqref{eqn: loss_func_l_less_k} $\mathcal{L}_{Test}^{(l,k,n)}(\mathcal{D}_{Test})$ converges to $H(Y_{l}|\Bar{g}(X_{l},\ldots,X_1))$ approximately as $\sum_{\Bar{g}(x_{k},\ldots,x_1)}p(\Bar{g}(x_k,\ldots,x_{1}))\mathcal{O}(\frac{1}{n p(\Bar{g}(x_k,\ldots,x_{1}))}) \approx \mathcal{O}(\frac{|\mathcal{S}|}{n})$.

\paragraph*{Case $l > k$}
Here the evaluation loss $\mathcal{L}_{Test}^{(l,k,n)}(\mathcal{D}_{Test})$ is given by
\begin{align} \label{eqn: loss_func_l_ge_k}
    \mathcal{L}_{Test}^{(l,k,n)}(\mathcal{D}_{Test}) &= H(Y_{k}|\Bar{g}(X_{k},\ldots,X_1)) \nonumber \\ 
    & + \sum_{(x_k,\ldots,x_0)}p(x_k,\ldots,x_1)\times\\  & D(p(y_k|\Bar{g}(x_{k},\ldots,x_1)) || p_{\theta_k}^n(y_k|\Bar{g}(x_{k},\ldots,x_1))),
\end{align}
The optimal value of $\mathcal{L}_{Test}^{(l,k,n)}(\mathcal{D}_{Test})$ is clearly $H(Y_{k}|\Bar{g}(X_k,\ldots,X_1)) = H(Y_{k}|g(X_k,\ldots,X_1)) = H(Y_{k}|X_k,\ldots,X_1)$ from non-negativity of KL divergence and the data generation process. Moreover, the convergence rate of $\mathcal{L}_{Test}^{(l,k,n)}(\mathcal{D}_{Test})$ to $H(Y_{k}|X_{k},\ldots,X_1)$ approximately with rate $\mathcal{O}(\frac{|\mathcal{S}|}{n})$, the same way as in the previous case. 
\end{proof}

\begin{proof}[Thm.~\ref{thm: optimal_limit}]
The proof follows directly from the non-negativity of KL divergence.
\paragraph{$l \leq k$} Here, we have
\begin{align*}
    &\sum_{(y_k, x_k,\ldots, x_1) \in \{0,1\}^k}p(y_k, x_k\ldots, x_1) \log{(h(y_k| x_k\ldots, x_1))}\nonumber \\
    &= H(Y_k|X_k, \ldots, X_1) + \sum_{(x_k,\ldots, x_1) \in \{0,1\}^k}p(y_k, x_k\ldots, x_1) \times\\
    & D(p(y_k| x_k\ldots, x_1)||h(y_k| x_k\ldots, x_1))\nonumber\\
    &= H(Y_l|X_l, \ldots, X_1) + \sum_{(x_k,\ldots, x_1) \in \{0,1\}^k}p(y_k, x_k\ldots, x_1)\\
    & D(p(y_k| x_k\ldots, x_1)||h(y_k| x_k\ldots, x_1))
\end{align*}
Note that $H(Y_k|X_k, \ldots, X_1) = H(Y_l|X_l, \ldots, X_1)$ because $Y_k$ only depends on the past $l$ observations and the stationarity assumption. Further, we have $D(\cdot||\cdot) \geq 0$ to conclude the proof.

\paragraph{$l > k$} Here, we have
\begin{align*}
    &\sum_{(y_k, x_k,\ldots, x_1) \in \{0,1\}^k}p(y_k, x_k\ldots, x_1) \log{(h(y_k| x_k\ldots, x_1))}\nonumber \\
    &= H(Y_k|X_k, \ldots, X_1) + \sum_{(x_k,\ldots, x_1) \in \{0,1\}^k}p(y_k, x_k\ldots, x_1)\\
    & D(p(y_k| x_k\ldots, X_1)||h(y_k| x_k\ldots, x_1)),\nonumber \\
    &\geq H(Y_k|X_k, \ldots, X_1)
\end{align*}
where we have $D(\cdot||\cdot) \geq 0$.
\end{proof}

\section{Understanding Transformers and Data-Efficient Training}\label{sec: finite_data_analysis}
First, we understand the role of various components of a Transformer in the prediction process. Then, we show the importance of relative position encodings~\cite{shaw2018self} and group-equivariant positional encodings~\cite{romero2020group} statistically in resource-constrained datasets. In particular, our results show that these methods are beneficial only when the available data is small, but their effect is negligible for larger datasets. For this analysis, we use position augmentations and group augmentations instead of equivariance for the ease of analysis. In practice, both equivariance and augmentation enforce the same constraint on the model. While equivariance applies the constraints by design in the model, augmentation does it more implicitly when training.

\subsection{Understanding Transformers}\label{subsec: understanding_transformers}
Here we discuss the role of various components of a Transformer.
\paragraph{Attention weights} Attention weights play an important role in filtering out unimportant data points and hence help make the range of the state function $\Bar{g}$ as small as possible while also ensuring that the important information is retained. We validate this claim by showing that when attention weights are used, additional irrelevant information does not affect the loss values. Whereas, without attention, the loss values are negatively affected.

On the other hand, in the extreme case where all the input variables are important, we show, perhaps surprisingly, that attention performs worse than a network without it.

\paragraph{Absolute positional encoding and $FFN$} 
Absolute positional encoding introduced by \citet{vaswani2017attention} helps identify the order of the words in a sentence for the Transformer. Without positional encoding, the output of a Transformer is equivariant to the permutation of words in the input. This is not desirable since changing the order of words in a sentence can completely change its meaning or may even make it gibberish. Hence, the understanding of the order of the words is important to a Transformer. 

In our experiments, we find that $FFN$ does not effectively filter out unnecessary information like attention. But it does improve the results when used with attention, indicating that it helps produce a better representation of the state of the Transformer.

\subsection{Data-Efficient Training: Positional Augmentation}\label{subsec: data_efficient_training}
Equivariance and augmentations are considered to be two sides of the same coin.  Whereas one applies constraints on the model, the other transforms the data during the training process to implicitly introduce those constraints in the learnt functions. Here, we use group augmentations~\cite{chen2020group} to study its impact on the performance of Transformers as it helps our theoretical analysis. Augmentations do not change the asymptotic performance limits, but, improves finite data performance as we find next.

\paragraph{Positional augmentation} We study the role of translation augmentation in the performance of Transformers. The key to our analysis is the following: we transform the data while keeping the position encodings the same and since the input to the Transformer is a summation of the data and the absolute encodings, the input appears as new data. This effectively increases the number of training data points in the dataset. Thus, for each augmentation, we count the effective size of the augmented dataset and analyze the performance from the finite data loss expression obtained in Thm.~\ref{thm: thm_convergence}.

For translation augmentation for a dataset of size $n$ and each batch processing $n_{pos}$ tokens at once. A translation augmentation parametrized by $t_0$, can be obtained by shifting the data by $t$, where $t$ goes from $t_0$ to $n_{pos} - 1$. This way, the semantics of text does not change, but we generate more data points. From the finitary analysis in Thm.~\ref{thm: thm_convergence}, we know the loss value converges with rate $O(1/n)$. Hence, with augmentation $t_0$, we have the following gain in the convergence rate. The proof is direct from the assumption that there are no repetitions in the dataset and counting the size of the augmented dataset.

\begin{proposition}\label{prop: aug}
Let a dataset of size $n$ be augmented by $t_0$ with $n_{pos}$ tokens processed per batch. Then the gain convergence rate is $O(\frac{1}{n} (1 - \frac{1}{(n_{pos} - t_0)}))$.
\end{proposition}

Two important consequences to note from Prop.~\ref{prop: aug} are: a) the gain decreases proportionally with increase in $n$, b) it increases with decrease in $t_0$.

\section{Further Experiments}\label{sec: experiments}
Here we validate our findings using various experiments.
\subsection{Datasets}\label{subsec: datasets}
We work with three sequential datasets: two are synthetic and one is a real language dataset, WikiText2~\cite{merity2016pointer}. The synthetic datasets are MarkovBoolSum and MarkovBin2Dec. MarkovBoolSum with Markovity $l$, is constructed by first uniformly generating a binary input, then the labels are set as 1 if the sum of previous $l$ inputs sum to greater than or equal to $l/2$, else 0. For MarkovBin2Dec, the input is generated in a similar way as MarkovBoolSum, but, the output is the decimal value corresponding to the binary observation of the $l$ previous observations. The main difference between the two datasets is that in MarkovBoolSum, not all the inputs contribute to the output, hence attention should help performance. Whereas, in MarkovBin2Dec, all the input values contribute to the output, hence, attention might not perform as well. All experiments were run on one fixed seed.

\subsection{Model Architecture}\label{subsec: model_architecture}
We use a single attention layer Transformer with varying attention span. The model dimension is fixed to $64$, and the hidden dimension for $FFN$ is set to $128$. For our experiments, we keep three design choices: a) attention or aggregation; b) $FFN$ or no $FFN$; c) translation augmentation.

\subsection{Attention}\label{subsec: exp_attn}
In Tab.~\ref{tab: attention}, we show results that indicate the benefit of attention. The synthetic datasets used $l=5$, and were trained for 10000 train batches for 2000 steps of batch size 20. It is found that using attention, the difference between varying lengths of span $k$ is less than without attention. Also, in the case $k=5$ for MarkovBin2Dec, we find that attention performs poorly compared to no attention. This is because, for $k=5$, all points in the input are important and attention may have neglected some of the inputs.
\begin{table}
\centering
\caption{Attention}
\label{tab: non-eq-acc}
\begin{tabular}{lllllll}
Dataset       &       & \multicolumn{3}{l}{Attention} & \multicolumn{2}{l}{Aggregation}  \\
k             & 5     & 10   & 15   & 5               & 10    & 15                       \\
MarkovBoolSum & 0.016 & 0.34 & 0.36 & 0.011           & 0.375 & 0.421                    \\
MarkovBin2Dec & 0.036 & 0.87 & 0.99 & 0.015           & 0.91  & 1.08                     \\
WikiText2     & 6.17  & 6.18 & 6.19 & 6.49            & 6.66  & 6.70                    
\end{tabular}
\label{tab: attention}
\end{table}
\subsection{Positional Encodings}\label{subsec: exp_pos_enc}

\begin{table}
\centering
\caption{Position augmentation}
\label{tab: non-eq-acc}
\begin{tabular}{lllll}
n           & \multicolumn{2}{l}{1} & \multicolumn{2}{l}{1000}  \\
$t_0$       & 0    & 99             & 0     & 99                \\
loss values & 9.68 & 0.87           & 0.628 & 0.624     
\end{tabular}
\label{tab: pos_aug}
\end{table}

The results for position augmentation are given in Tab.~\ref{tab: pos_aug}, where we find that augmentation helps improve test loss performance. But, the loss is  significant only when the number of train batches are very small, corresponding to our theoretical analysis.


\end{document}